\documentclass[12pt,reqno]{amsart}
\usepackage[foot]{amsaddr}
\usepackage{amsmath}
\usepackage{amssymb}
\usepackage{verbatim}
\usepackage{bbm}
\usepackage{graphicx}
\usepackage[font=footnotesize]{caption}
\usepackage{subcaption}
\usepackage{mathrsfs}
\usepackage{epstopdf}
\usepackage{multirow}
\usepackage{url}
\usepackage{makecell}
\usepackage{tabularx}
\usepackage{longtable}
\usepackage{graphicx}
\usepackage{tikz}
\usepackage{pgfplots}
\usepackage{xcolor}
\usepackage{pgfplots}
\usepackage{pgfplotstable}
\allowdisplaybreaks

\usepackage[toc]{appendix}
\usepackage[capitalise]{cleveref}
\usepackage{color}
\allowdisplaybreaks

\usepackage[top=0.5in,bottom=0.5in,left=0.7in,right=0.7in]{geometry}

\newtheorem{lemma}{Lemma}

\newtheorem{theorem}[lemma]{Theorem}
\newtheorem{remark}[lemma]{Remark}

\newtheorem{example}[lemma]{Example}

\newcommand{\N}{\mathbb{N}}    
\newcommand{\R}{\mathbb{R}}    
\newcommand{\LL}{\mathbf{\mathcal{L}}}
\newcommand{\totalneurons}{10889 d + 10887}

\newcommand{\tp}{^{\mathsf{T}}}

\DeclareMathOperator*{\argmax}{arg\,max}

\newcommand{\be}{ \begin{equation} }
	\newcommand{\ee}{ \end{equation} }

\numberwithin{equation}{section}
\numberwithin{lemma}{section}

\begin{document}

\title{Optimal Neural Network Approximation for High-Dimensional Continuous Functions}

\author{Ayan Maiti}
\address{Department of Mathematics, Purdue University, West Lafayette, IN, USA 47907.}
\email{maitia@purdue.edu}

\author{Michelle Michelle}
\address{Department of Mathematical and Statistical Sciences, University of Alberta, Edmonton, Alberta, Canada T6G 2G1.}
\email{mmichell@purdue.edu, mmichell@ualberta.ca}

\author{Haizhao Yang}
\address{Department of Mathematics, Department of Computer Science (Affiliated), The University of Maryland Institute for Advanced Computer Studies (Affiliated), University of Maryland, College Park, MD, USA 20742.}
\email{hzyang@umd.edu}


\makeatletter \@addtoreset{equation}{section} \makeatother
	
\begin{abstract}
	Recently, the authors of \cite{SYZ22} developed a neural network with width $36d(2d + 1)$ and depth $11$, which utilizes a special activation function called the elementary universal activation function, to achieve the super approximation property for functions in $C([a,b]^d)$. That is, the constructed network only requires a fixed number of neurons (and thus parameters) to approximate a $d$-variate continuous function on a $d$-dimensional hypercube with arbitrary accuracy. More specifically, only $\mathcal{O}(d^2)$ neurons or parameters are used. One natural question is whether we can reduce the number of these neurons or parameters in such a network. By leveraging a variant of the Kolmogorov Superposition Theorem, \textcolor{black}{we show that there is a composition of networks generated by the elementary universal activation function with at most $\totalneurons$ nonzero parameters such that this super approximation property is attained. The composed network consists of repeated evaluations of two neural networks: one with width $36(2d+1)$ and the other with width 36, both having 5 layers.}  Furthermore, we present a family of continuous functions that requires at least width $d$, and thus at least $d$ neurons or parameters, to achieve arbitrary accuracy in its approximation. This suggests that the number of nonzero parameters is optimal in the sense that it grows linearly with the input dimension $d$, unlike some approximation methods where parameters may grow exponentially with $d$. 
\end{abstract}


\maketitle

\pagenumbering{arabic}


\section{Introduction}

	The wide applicability of neural networks has generated tremendous interest, leading to many studies on their approximation properties. Some early work on this subject can be traced back to \cite{C89,HSW89}. As summarized in \cite{SYZ22}, there have been several research paths in this area such as finding nearly optimal asymptotic approximation errors of ReLU networks for various classes of functions \cite{HY21,PV18,Y18, YZ20}, deriving nearly optimal non-asymptotic approximation errors for continuous and $C^{s}$ functions \cite{LSYZ21, SYZ20}, mitigating the curse of dimensionality in certain function spaces \cite{B93, EW, MYD21}, and improving the approximation properties by using a combination of activation functions and/or constructing more sophisticated ones \cite{MP99, SYZ21, SYZ22, Y21,YZ20}.
	
	Building on the last point, \cite{Y21} presented several explicit examples of superexpressive activation functions, which if used in a network allows us to approximate a $d$-variate continuous function with a fixed architecture and arbitrary accuracy. That is, the number of neurons remains the same, but the values of the parameters may change. This approach is notably different from using a standard network with commonly used activation functions such as ReLU. To achieve the desired accuracy, a standard network with a commonly used activation function typically needs to have its width and/or its depth increased based on the target accuracy. The growth of the number of neurons in terms of the target accuracy may range from polynomial to, in the worst-case scenario, exponential. 
	
	The existence of a special activation function mentioned earlier has been known since the work of \cite{MP99}; however, its explicit form is unknown, even though the activation itself has many desirable properties such as sigmoidal, strictly increasing, and analytic. The results in \cite[Theorem 5.22 and Section 5.4.1]{I21} improved \cite{MP99} by reducing the number of neurons and using an activation function that can be evaluated via an algorithm. In the same vein, \cite{SYZ22} introduced several new explicit activation functions, called universal activation functions, that allow a network with a fixed architecture to achieve arbitrary accuracy when approximating a $d$-variate continuous function. In a follow-up work, \cite{{Wetal24}} presented additional examples of universal activation functions and evaluated their performance on various datasets.
	
	More recently, there have been studies that look into the minimum required width of networks generated by various activation functions to achieve the universal approximation property for functions in $L_p$ spaces and continuous functions \cite{C23, HS18, J19, KL20, LDJC23, LC24, PYLS21}. We briefly review relevant results for continuous functions. Suppose that $K$ is a compact domain in $\R^{d_{x}}$ and let $w_{\min}$ denote the minimum width. The authors of \cite{HS18} found that a ReLU network requires $w_{\min} = d_{x} + 1$ for functions in  $C(K,\R)$ (or equivalently, $C(K)$). More generally, for functions in $C(K,\R^{d_{y}})$, a ReLU+STEP network requires $w_{\min} = \max(d_{x}+1,d_{y})$ \cite{PYLS21}, a network with an arbitrary activation function requires $w_{\min} \ge \max(d_{x}, d_{y})$ \cite{C23}, and a ReLU+FLOOR network requires $w_{\min} \ge \max(d_{x},d_{y},2)$ \cite{C23}. Additionally, the author of \cite{G03} showed that a network with an activation function that is twice differentiable at least at a point with its second derivative being nonzero requires width $d_x + d_y + 2$.
	If we consider functions in $C([0,1]^{d_{x}},\R^{d_{y}})$, then a network generated by an activation function that can be approximated by a sequence of continuous one-to-one continuous functions requires $w_{\min} \ge d_{x} +1$ \cite{J19}, a network generated by a non-polynomial activation function that is continuously differentiable at least a point requires $w_{\min} \le d_{x} + d_{y} + 1$ \cite{KL20}, and a network with a non-affine polynomial activation function requires $w_{\min} \le d_{x} + d_{y} + 2$ \cite{KL20}. 

	To approximate a $d$-variate continuous function on a $d$-dimensional hypercube, some network constructions \cite{LS23, MYD21, SYZ22} rely on the Kolmogorov Superposition Theorem (KST) \cite{K57}. KST represents such a function in terms of compositions and additions of univariate continuous functions on bounded intervals, thereby making the analysis of such a function highly convenient.
	
	The present paper is motivated by the findings of \cite{SYZ22}. Their network has width $36d(2d+1)$ and depth $11$, and is capable of approximating functions in $C([a,b]^d)$ with arbitrary accuracy. The authors used the original KST \cite{K57} to convert the analysis of a $d$-variate continuous function into that of several univariate continuous functions. Furthermore, they constructed an elementary universal activation function (EUAF) network to approximate a univariate continuous function with arbitrary accuracy. A natural question is whether the same super approximation property can be achieved with fewer neurons or parameters. 
	
	The main contributions of this paper are twofold. Firstly, we show that there is a composition of EUAF networks with at most $\totalneurons$ nonzero parameters achieving the desired super approximation property. That is, we can approximate a target function in $C([a,b]^d)$ with arbitrary accuracy using at most $\totalneurons$ nonzero parameters. \textcolor{black}{The proposed network is nonstandard in that it requires repeated evaluations of two neural networks: one with width $36(2d+1)$ and the other with width 36, both having 5 layers.} To obtain a better approximation, only the values of these parameters change. This is in stark contrast to a standard network generated by commonly used activation functions such as ReLU, where the number of parameters typically grows to obtain a more accurate approximation. The network in \cite{SYZ22} requires $\mathcal{O}(d^2)$ neurons or parameters, because it relies on the original KST, which has $2d+1$ outer functions and $(2d+1)d$ inner functions. Using a variant of KST (\cite{K75} or \cite[Theorem 5]{MP99}) allows us to only use 1 outer function and $2d+1$ inner functions. Therefore, we only need to approximate these $2d+2$ functions by EUAF networks once and evaluate them repeatedly (see \cref{fig:compflow}). Not only can we reduce the number of nonzero parameters to $\mathcal{O}(d)$, but the proof is simplified. Secondly, we present a family of continuous functions that requires at least width $d$, and thus at least $d$ neurons or parameters to achieve arbitrary accuracy in its approximation. These results suggest that the number of nonzero parameters for approximating functions in $C([a,b]^d)$ is optimal in that it linearly depends on the input dimension $d$. This requirement is significantly less severe than some other approximation methods, which may use an exponentially growing number of parameters. To better understand how our study compares to others in terms of the width and depth requirements, we refer readers to \cref{fig:depthwidth}.
	
	The organization of this paper is as follows. In \cref{sec:prelim}, we review some basic notations and key ingredients used in the proof of our main results. In \cref{sec:main}, we show the existence of a composition of EUAF networks with at most $\totalneurons$ nonzero parameters approximating functions in $C([a,b]^d)$ with arbitrary accuracy. Additionally, we present a family of continuous functions that requires at least width $d$ (or $d$ neurons) for its network approximation to achieve arbitrary accuracy.  

	\begin{figure}[htbp]
		\includegraphics[width=\textwidth]{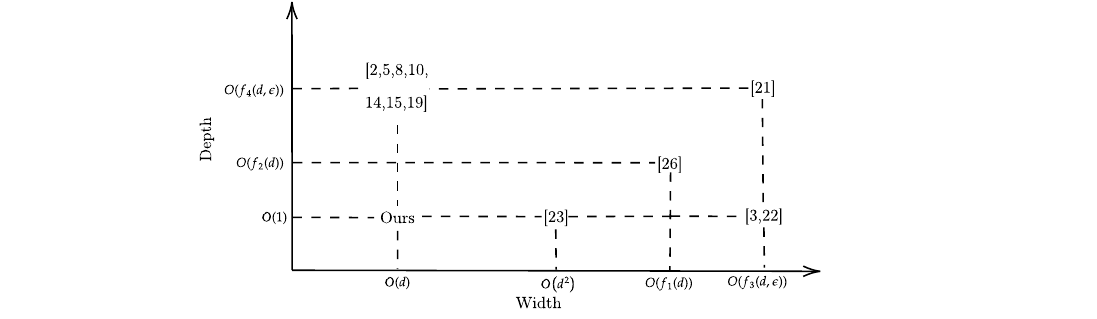}
		\caption{A diagram showing relevant studies discussing requirements for achieving a target approximation accuracy $\epsilon$ in terms of the width, depth, and the input dimension $d$ for a $d$-variate continuous function on a $d$-dimensional hypercube with a scalar output \textcolor{black}{and activation functions with explicit expressions}. Studies with width $\mathcal{O}(f_3(d,\epsilon))$ and/or depth $\mathcal{O}(f_4(d,\epsilon))$ imply that achieving a more accurate approximation requires increasing the depth and/or width based on $d$ and $\epsilon$ (and possibly other factors like the modulus of continuity of the target function). Our paper and \cite{SYZ22, Y21}  
		have the super approximation property in the sense that only a fixed number of parameters is needed to achieve arbitrary accuracy. The network constructed in \cite{Y21} has width $\mathcal{O}(f_1(d))$ and depth $\mathcal{O}(f_2(d))$, where $f_1$ and $f_2$ are some functions that are not known explicitly. \textcolor{black}{Our proposed network has a nonstandard representation, as it involves repeated evaluations of two neural networks with width $\mathcal{O}(d)$; more specifically, one with width $36(2d+1)$ and the other with width $36$, both having depth 5.  See \cref{fig:compflow} for our network.}}
		\label{fig:depthwidth}
	\end{figure}

\section{Preliminaries} \label{sec:prelim}
	For a given activation function $\sigma$, the function $\phi: \R^{d_{\text{in}}} \rightarrow \R^{d_{\text{out}}}$ is a $\sigma$ \textcolor{black}{network} with $L \in \N$ layers if 
	\be \label{NN}
	\phi := \LL_{L} \circ \sigma \circ \LL_{L-1} \circ \sigma \circ \dots \circ \sigma \circ \LL_{1} \circ \sigma \circ \LL_{0},
	\ee
	where $\LL_i(\mathbf{y}_i):=\mathbf{W}_{i} \mathbf{y}_i + \mathbf{b}_i$ with a weight matrix $\mathbf{W}_{i} \in \R^{N_{i+1} \times N_{i}}$, $\mathbf{y}_i \in \R^{N_{i}}$ with $\mathbf{y}_i=(y_1,\ldots,y_{N_{i}})\tp$, a bias vector $\mathbf{b}_i \in \R^{N_{i+1}}$, $N_0=d_{\text{in}}$, $N_{L+1}=d_{\text{out}}$, and the activation function is applied elementwise (i.e., $\sigma(\mathbf{y}) := (\sigma(y_1),\ldots,\sigma(y_{N_{i}}))\tp$). If $N_i=N$ for all $1\le i \le L$ (i.e., there are $N$ neurons for each hidden layer), then we say that such a $\sigma$ \textcolor{black}{network} has width $N$, depth $L$, and $N \times L$ neurons with the total number of parameters equaling to $\sum_{i=0}^{L} (N_{i+1} N_i + N_{i+1})$.

	There are many available activation functions in the literature. We are particularly interested in the EUAF introduced in \cite{SYZ22}, which is defined as 
	\be \label{EUAF}
	\sigma(x) := 
	\begin{cases}
		\left|x - 2 \lfloor \tfrac{x+1}{2} \rfloor\right|, & x \in [0,\infty), \\
		\frac{x}{|x|+1}, & x \in (-\infty,0). \\
	\end{cases}
	\ee
	As can be seen from its plot in \cref{fig:EUAFpic}, the function $\sigma$ resembles a softsign function for all $x<0$ and a periodic hat function for all $x>0$. In this paper, we shall focus on EUAF networks (i.e., $\sigma$ networks with $\sigma$ defined in \eqref{EUAF}).
		
	\begin{figure}[htbp]
		\includegraphics[width=0.3\textwidth]{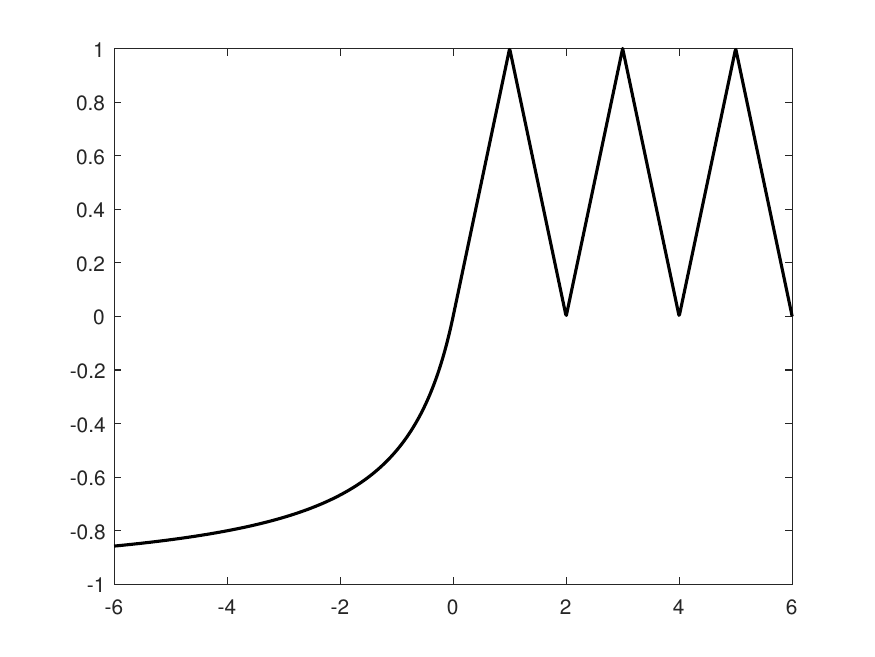}
		\caption{The plot of the EUAF in \eqref{EUAF}.}
		\label{fig:EUAFpic}
	\end{figure}
		
		
	We first present the two key ingredients in the construction an EUAF \textcolor{black}{network} with at most $\totalneurons$ nonzero parameters that can approximate any function in $C([a,b]^d)$ with arbitrary accuracy. The first ingredient is a version of KST \cite{K57}, which was studied in \cite{K75} and utilized in the construction of the network in \cite{MP99}. In contrast to the original KST, the following version requires only one outer function and $2d+1$ inner functions, which enables us to use only $\totalneurons$ nonzero parameters.
	\begin{theorem} \label{KST}
		(\cite{K75} or \cite[Theorem 5]{MP99}) Let $\mathbf{x} = (x_1,\dots,x_d)^{\tp}$. There exist $d$ constants $\lambda_j>0$, $j=1,\ldots,d$, $\sum_{j=1}^{d} \lambda_j \le 1$, and $2d+1$ continuous strictly increasing functions $h_i$, $1 \le i \le 2d+1$, which map $[0,1]$ to itself, such that every continuous function $f$ of $d$ variables on $[0,1]^d$ can be represented in the form 
		\[
		f(\mathbf{x}) = \sum_{i=1}^{2d+1} g\left( \sum_{j=1}^{d} \lambda_j h_i(x_j) \right)
		\]
		for some $g \in C([0,1])$ depending on $f$.
	\end{theorem}
	The second ingredient is the following result from \cite{SYZ22} on the existence of an EUAF network with fixed width and depth that can approximate any function in $C([a,b])$ with arbitrary accuracy.
	\begin{theorem} \label{EUAF1D}
		\cite[Theorem 6]{SYZ22} Let $f \in C([a,b])$. Then, for an arbitrary $\epsilon >0$, there exists a function $\phi$ generated by an EUAF network with width 36 and depth 5 such that 
		\[
		|\phi(x) - f(x)| < \epsilon \quad \text{for any } x \in [a,b] \subseteq \mathbb{R}.
		\]
	\end{theorem}
	The above theorem indicates that we have $72 + 4 \times (36^2 + 36) + 37 = 5437$ parameters, since $N_0=N_5=1$ and $N_i=36$ for all $1\le i \le 4$. As outlined in \cite{SYZ22}, the construction of such an EUAF network was performed by using a three-step procedure: (1) divide the bounded interval into several uniform sub-intervals (the number of these smaller intervals depends on the prescribed error and the target function), (2) enumerate them and build a sub-network that maps each sub-interval to its index, and (3) build another sub-network that maps the index of the sub-interval to the target function evaluated at the left endpoint of the sub-interval. 	
		
	There is also a possibility of further reducing the number of nonzero parameters in the network by combining the EUAF with superexpressive activation functions presented in \cite{Y21}. However, for simplicity, we choose to use the same activation function throughout the entire network and adhere to Theorem \ref{EUAF1D}. Moreover, in the context of techniques used in the paper, reducing the number of nonzero parameters in the approximation of a function in $C([a,b])$ will not lead to a reduction in the order of magnitude with respect to $d$, when combining this result with the KST to approximate a function in $C([a,b]^d)$.
		
	\section{Main results} \label{sec:main}
	Now, we are ready to present our first main result. The following theorem guarantees the existence of a composition of EUAF networks with $\totalneurons$ nonzero parameters that can approximate any function in $C([a,b]^d)$ with arbitrary accuracy. 
	\begin{theorem} \label{thm:uap}
		Let $f \in C([a,b]^d)$. Then, for an arbitrary $\epsilon>0$, there exists a function $\phi$ generated by \textcolor{black}{a composition of EUAF networks} with at most $\totalneurons$ nonzero parameters such that 
		\[
		|f(\mathbf{x}) - \phi(\mathbf{x})| < \epsilon \quad \text{for all } \mathbf{x} \in [a,b]^d.
		\]
	\end{theorem}	
		
	\begin{proof}
		Let $\mathbf{x}=(x_1,\ldots,x_d)^{\textsf{T}}$ and $\mathbf{y}=(y_1,\ldots,y_d)^{\textsf{T}}$. Define $\tilde{\mathcal{L}}(t):=a + (b-a) t$ for $t \in [0,1]$ and $\tilde{f}(y_1,\dots,y_d):= f(\tilde{\mathcal{L}}(y_1),\dots,\tilde{\mathcal{L}}(y_d))$ for all $(y_1,\ldots,y_d)^{\textsf{T}} \in [0,1]^d$. Clearly,  $\tilde{f} \in C([0,1]^d)$. By Theorem \ref{KST}, we have 
		\begin{equation} \label{ftil}
		\tilde{f}(\mathbf{y}) = \tilde{f}(y_1,\ldots,y_d) = \sum_{i=1}^{2d+1} g \left( \sum_{j=1}^{d} \lambda_j \tilde{h}_i(y_j) \right), 
		\quad y_j \in [0,1] \text{ for all } 1\le j \le d,
		\end{equation}
		where $\sum_{j=1}^d \lambda_j \le 1$ with $\lambda_j >0$ for each $j$, each $\tilde{h}_i$ is a continuous strictly increasing function mapping the interval $[0,1]$ to itself, and $g \in C([0,1])$. If we define $\mathcal{L}(t):= (t-a)/(b-a)$ for $t \in [a,b]$, then \eqref{ftil} yields 
		\[
		f(x_1,\ldots,x_d) = f(\tilde{\mathcal{L}}(y_1),\ldots,\tilde{\mathcal{L}}(y_d)) = \tilde{f}(y_1,\ldots,y_d) = \sum_{i=1}^{2d+1} g \left( \sum_{j=1}^{d} \lambda_j \tilde{h}_i(y_j) \right) = \sum_{i=1}^{2d+1} g \left( \sum_{j=1}^{d} \lambda_j h_i(x_j) \right),
		\]
		where $h_i:=\tilde{h}_i \circ \mathcal{L}$ and $x_j \in [a,b]$ for all $i,j$. Note that each $h_i$ is now a continuous function mapping the interval $[a,b]$ to $[0,1]$. 
		
		Now, arbitrarily fix $\epsilon >0$. First, we focus on the approximation of the outer function $g$. Since $g$ is a uniformly continuous function on $[0,1]$, we know that there exists $\delta >0$ such that 
		\begin{equation} \label{guc}
		|g(z_1) - g(z_2)| < \frac{\epsilon}{2(2d+1)} \quad \text{for all } z_1, z_2 \in [0,1] \text{ with } |z_1 - z_2| < \delta.
		\end{equation}
		Additionally, by Theorem \ref{EUAF1D}, we know that there is an EUAF network $\tilde{\phi}$ with width 36 and depth 5 such that 
		\begin{equation} \label{gphi}
		|g(z) - \tilde{\phi}(z)| < \frac{\epsilon}{2(2d+1)} \quad \text{for all } z \in [0,1].
		\end{equation}
		Next, we turn to the approximation of each inner function $h_i$. For each $i$, we know by Theorem \ref{EUAF1D} again that there is an EUAF network $\tilde{\psi}_i$ with width $36$ and depth $5$ such that 
		\be \label{diffhz}
		|h_i(z) - \tilde{\psi}_i(z)| < \delta \quad \text{for all } z \in [a,b]. 
		\ee
		Define $\psi_i:=\min\{\max\{\tilde{\psi}_i,0\},1\}$. If $\tilde{\psi}_{i}(z) < 0$ for some $z \in [a,b]$, then 
		\[
		h_i(z) - \psi_{i}(z) < h_i(z) - \tilde{\psi}_{i}(z) < \delta,
		\]
		since $h_i$ is a strictly increasing continuous function whose range is contained in $[0,1]$. Otherwise, if $\tilde{\psi}_{i}(z) >1$ for some $z \in [a,b]$, then 
		\[
		\psi_{i}(z) - h_{i}(z) < \tilde{\psi}_{i}(z)-h_{i}(z) < \delta
		\]
		due to the same reason. Thus, we have 
		\[
		|h_i(z) - \psi_i(z)| \le |h_i(z) - \tilde{\psi}_i(z)| < \delta  \quad \text{for all } z \in [a,b]. 
		\]
		By \eqref{EUAF}, we observe that 
		\[
		\min\{\max\{t,0\},1\} = \frac{1}{2}\left((t+1) - \sigma(t+1)\right) = \frac{3}{2}\sigma\left(\frac{1}{3}t + \frac{1}{3}\right) - \frac{1}{2}\sigma(t+1) \quad \text{for all } t \in[-1,2],
		\]
		which implies that $\psi_i$ can be constructed by adding 6 more parameters to further process the output of $\tilde{\psi}_i$. 
		%
		%
		See \cref{fig:psitil} for a visualization of the network $\psi_i$. Since $\sum_{j=1}^{d} \lambda_j = 1$ with $\lambda_j>0$ and the range of $\psi_i$ is contained in $[0,1]$ for all $1 \le i \le 2d+1$, we can immediately see that the range of $\sum_{j=1}^{d} \lambda_j \psi_i(x_j)$ is contained in $[0,1]$.
		
		\begin{figure}[htbp]
			\includegraphics[width=\textwidth]{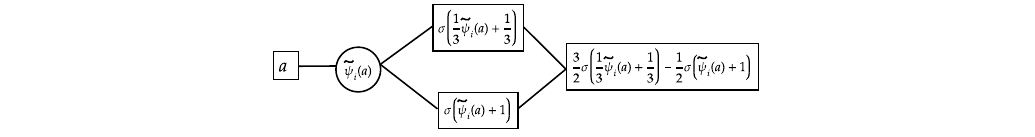}
			\caption{The network $\psi_i$. Note that $\tilde{\psi}_i$ has 5437 parameters as discussed in the remark following Theorem \ref{EUAF1D}. Since we need to add 6 more parameters to further process the output of $\tilde{\psi}_i$, the total number of parameters is $5443$.}
			\label{fig:psitil}
		\end{figure}
		
		Define
		\[
		\phi(\mathbf{x}) := \sum_{i=1}^{2d+1} \tilde{\phi} \left( \sum_{j=1}^{d} \lambda_j \psi_i(x_j) \right), 
		\quad \mathbf{x} \in [a,b]^d.
		\]
		By \eqref{diffhz}, we observe that 
		\be \label{gucdelta}
		\left| \sum_{j=1}^{d} \lambda_j h_i(x_j) - \sum_{j=1}^{d} \lambda_j \psi_i(x_j)\right| \le \sum_{j=1}^{d} \lambda_j |h_i(x_j) - \psi_i(x_j)| < \sum_{j=1}^{d} \lambda_j \delta = \delta,
		\quad 1\le i \le 2d+1.
		\ee
		Therefore, for $\mathbf{x} \in [a,b]^d$, we have 
		\begin{align*}
			& |f (\mathbf{x}) - \phi(\mathbf{x})| = \left| \sum_{i=1}^{2d+1} g \left( \sum_{j=1}^{d} \lambda_j h_i(x_j) \right) - \sum_{i=1}^{2d+1} \tilde{\phi} \left( \sum_{j=1}^{d} \lambda_j \psi_i(x_j) \right) \right|\\
			& \le \left| \sum_{i=1}^{2d+1} g \left( \sum_{j=1}^{d} \lambda_j h_i(x_j) \right) - \sum_{i=1}^{2d+1} g \left( \sum_{j=1}^{d} \lambda_j \psi_i(x_j) \right)\right| + \left| \sum_{i=1}^{2d+1} g \left( \sum_{j=1}^{d} \lambda_j \psi_i(x_j) \right) - \sum_{i=1}^{2d+1} \tilde{\phi} \left( \sum_{j=1}^{d} \lambda_j \psi_i(x_j) \right) \right| \\
			& \le \sum_{i=1}^{2d+1} \left| g \left( \sum_{j=1}^{d} \lambda_j h_i(x_j) \right) - g \left( \sum_{j=1}^{d} \lambda_j \psi_i(x_j) \right)\right| + \sum_{i=1}^{2d+1} \left| g \left( \sum_{j=1}^{d} \lambda_j \psi_i(x_j) \right) - \tilde{\phi} \left( \sum_{j=1}^{d} \lambda_j \psi_i(x_j) \right) \right|\\
			& < \frac{\epsilon}{2(2d+1)} (2d+1) +  \frac{\epsilon}{2(2d+1)} (2d+1) = \epsilon,
		\end{align*}
		where we applied \eqref{guc} to the first term of the last inequality (since \eqref{gucdelta} holds) and \eqref{gphi} to the second term of the last inequality. 
		
		Finally, we count the number of nonzero parameters used in $\phi(\mathbf{x})$. 
		We first apply the sub-network $\psi_i$, where $1\le i \le 2d+1$, repeatedly to each $x_j$. This implies that we require at most $5443(2d+1)$ nonzero parameters. For a given $i$, we take a linear combination of $\psi_i$ evaluated at all $x_j$, where $1 \le j \le d$. This operation requires $d$ parameters. Afterwards, we apply the sub-network $\tilde{\phi}$ repeatedly to each aforementioned output, which involves $5443$ parameters. Finally, we add up all outputs which requires $2d+1$ parameters.  Therefore, the total number of nonzero parameters is at most $5443(2d+1) + d + 5443 + 2d+1 = \totalneurons$. The proof is completed.   
		%
	\end{proof}
	
	 

	\begin{remark} \label{remark:thm}
		\normalfont
		\textcolor{black}{The composition of EUAF networks described in Theorem \ref{thm:uap} can be viewed in \cref{fig:compflow}. Define $\psi:=(\psi_1,\dots,\psi_{2d+1})\tp$ (i.e., we vertically stack $\psi_i$ for all $i=1,\dots,2d+1$). This operation preserves the number of nonzero parameters (the new weight matrices can be obtained by forming a block diagonal matrix of weight matrices from individual networks, while the new biases can be formed by vertically stacking biases from individual networks). We sequentially evaluate $\psi$ at the points $x_j$ for all $1 \le j \le d$ producing $d$ vectors of length $2d+1$. Afterwards, we rearrange this set of vectors to a new set of $2d+1$ vectors of length $d$ and compute the weighted average of entries in each vector with weights $\lambda_j$, $1 \le j \le d$. At this point, we have $2d+1$ scalars, which we sequentially feed into $\tilde{\phi}$ and get $2d+1$ scalars as outputs. As the final step, we add all these $2d+1$ scalars up to obtain our final output. We emphasize again that $\psi$ and $\tilde{\phi}$ are EUAF networks with width $36(2d+1)$ and $36$ respectively, both with depth 5. The proposed network has a nonstandard representation due to the sequential and repeated evaluations of $\psi$ and $\tilde{\phi}$.}
	\end{remark}

	\begin{figure}[htbp]
		\includegraphics[width=\textwidth]{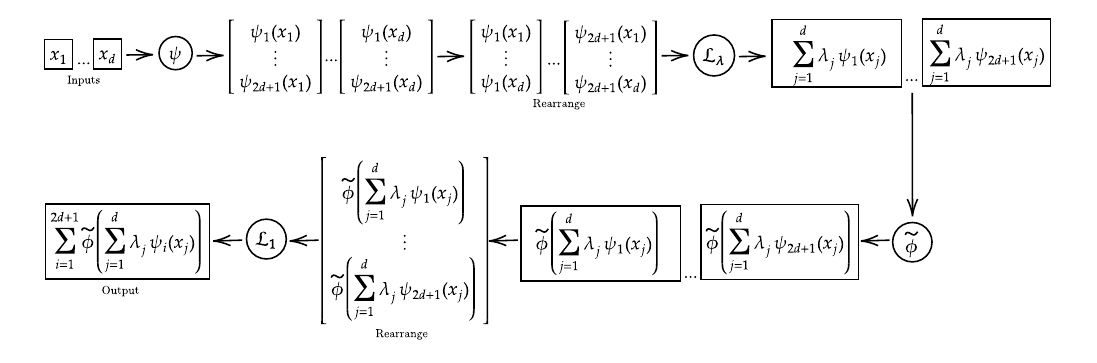}
		\caption{\textcolor{black}{The composition of EUAF networks mentioned in Theorem~\ref{thm:uap}, where $\psi:=(\psi_1,\dots,\psi_{2d+1})\tp$, and $\mathcal{L}_{\pmb{\lambda}}$, $\mathcal{L}_{\mathbf{1}}$ respectively represent an inner product with the vector $(\lambda_1,\dots,\lambda_d)$ and the vector of ones to obtain a scalar output. See Remark \ref{remark:thm}.}}
		\label{fig:compflow}
	\end{figure}

	Note that the EUAF network in \cite[Theorem 1]{SYZ22} has width $36d(2d+1)$ and depth $11$ with a total of $5437(d + 1)(2d + 1)$ parameters, because the version of KST used in their proof requires $(2d+1)d$ inner functions and $2d+1$ outer functions. Employing the version of KST in Theorem \ref{KST} not only allows us to use at most $\totalneurons$ nonzero parameters, but it also simplifies the proof of the existence of an EUAF network with a fixed architecture that can approximate any function in $C([a,b]^d)$ with arbitrary accuracy. Even though the total number of nonzero parameters in Theorem \ref{thm:uap} is larger than that in \cite[Theorem 4]{MP99}, its order of magnitude is the same, $\mathcal{O}(d)$, and the activation function in our network is explicitly known.

	Next, we present a family of continuous functions that requires at least width $d$ (or to put differently, at least $d$ neurons/parameters) for it to be approximated with arbitrary accuracy. In the following, we assume that the depth is fixed. 
	
	\begin{theorem} \label{example}
		Let $f \in C([-\tfrac{1}{2},\tfrac{1}{2}]^d)$ such that $f(\mathbf{0}) = 0$ (i.e., it vanishes at the origin) and \textcolor{black}{$|f(\mathbf{x})| = |f(x_1,\dots,x_d)| \ge c$ for some $c>0$ if $x_j=\tfrac{1}{2}$ for some $1\le j \le d$ (i.e., $|f(\mathbf{x})|$ is bounded away from zero if at least one of its inputs is equal to $\tfrac{1}{2}$)}. Then, for any given activation function $\sigma$, the $\sigma$ network with width less than $d$, and fixed depth $L\ge 1$ cannot approximate $f$ with arbitrary accuracy.
	\end{theorem}	
	
	\begin{proof}
		We use a proof by contradiction. Assume that for each $\epsilon>0$, there is a $\sigma$ network $\phi$ with width $d-1$ and fixed depth $L\ge 1$ such that
		\begin{equation} \label{ua3}
			\left|f(\mathbf{x}) - \phi(\mathbf{x}) \right| < \epsilon \quad \text{for all } \mathbf{x} \in [-\tfrac{1}{2},\tfrac{1}{2}]^d,
		\end{equation}
		where $\phi$ is defined as in \eqref{NN} with $N_i=d-1$ for $1\le i \le L$. More explicitly, 
		\be \label{phiW0}
		\phi(\mathbf{x}) = \LL_{L}(\sigma(\LL_{L-1}(\sigma(\dots \LL_{1}(\sigma(\mathbf{W}_0 \mathbf{x}+ \mathbf{b}_0))\dots)))),
		\ee
		where $\mathbf{W}_0 \in \R^{(d-1) \times d}$ and $\mathbf{b}_0 \in \R^{d-1}$. 
		
		Arbitrarily fix $\epsilon>0$. Consider the homogeneous linear system $\mathbf{W}_0 \tilde{\mathbf{x}} = \mathbf{0}$, where $\tilde{\mathbf{x}} := (\tilde{x}_1,\dots,\tilde{x}_d)\tp$. Such a system clearly has infinitely many solutions. Next, define $B:=\LL_{L}(\sigma(\LL_{L-1}(\sigma(\dots \LL_{1}(\sigma(\mathbf{b}_0))\dots))))$. Suppose that $|B| > \epsilon$. Letting $\mathbf{x}=\mathbf{0}$ in \eqref{ua3}, we have 
		\be \label{contr}
		\begin{aligned}
		\epsilon > \left|f(\mathbf{0}) - B \right| \ge \left||f(\mathbf{0})| - |B| \right| = |B|,
		\end{aligned}
		\ee
		where we used our assumption that $f(\mathbf{0})=0$. We obtain $\epsilon > |B| > \epsilon$, which is a contradiction. Now, suppose that $|B|<\epsilon$. Pick any nontrivial solution $\tilde{\mathbf{x}}$, define 
		\[ \hat{\mathbf{x}} := \frac{\text{sign}(\argmax_{1\le i \le d} |\tilde{x}_{i}|)}{2 \max_{1\le i \le d} |\tilde{x}_{i}|} \tilde{\mathbf{x}}.
		\]
		Clearly, $\hat{\mathbf{x}} \in [-\tfrac{1}{2},\tfrac{1}{2}]^{d}$ and at least one of its component is equal to $\tfrac{1}{2}$. Suppose that $|B|<\epsilon$. We have
		\[
		\epsilon > ||f(\hat{\mathbf{x}})| - |B|| \ge |f(\hat{\mathbf{x}})| - |B| \ge |f(\hat{\mathbf{x}})|- \epsilon \ge c - \epsilon, 
		\]
		where we used our assumption that $|f(\mathbf{x})| = |f(x_1,\dots,x_d)| \ge c$ for some $c>0$ if $x_j=\tfrac{1}{2}$ for some $1\le j \le d$. Therefore, we have a contradiction. The proof is completed. 
	\end{proof}
	
	The assumption on the support of the function $f$ can be generalized to $[a,b]^d$, where $a < 0 < b$. We can still apply the same argument by imposing analogous conditions on the function $f$ at $x_j = b$. Finally, we provide a concrete example of a function satisfying conditions in the above theorem.
	
	\begin{example}
		\normalfont 
		\textcolor{black}{Let $f(\mathbf{x})=\sum_{j=1}^{d} c_j h_j(x_j)$, where for all $1\le j \le d$, $c_j>0$, $x_j \in [-\tfrac{1}{2},\tfrac{1}{2}]$, and $h_j$ is a nonnegative continuous function such that $h_j(0)=0$ and $h_j(\tfrac{1}{2})\neq0$. Clearly, $f \in C([-\tfrac{1}{2},\tfrac{1}{2}]^d)$, $f(\mathbf{0})=0$, and $|f(\mathbf{x})| = |f(x_1,\dots,x_d)| \ge \left(\min_{1 \le j \le d} c_j \right) \left( \min_{1 \le j \le d} h_j(\tfrac{1}{2})\right)$ if $x_j = \tfrac{1}{2}$ for some $1 \le j \le d$.} The above theorem states that for any given activation function $\sigma$, the $\sigma$ network with width less than $d$, and fixed depth $L\ge 1$ cannot approximate $f$ with arbitrary accuracy.
	\end{example}

	Theorem \ref{example} presents a family of continuous functions, which cannot be approximated with arbitrary accuracy, when we let the width to be $d-1$ and fix the depth to be $L\ge 1$. This implies that the $\sigma$ network actually requires at least width $d$ and consequently at least $d$ neurons/parameters to achieve the desired approximation property. Theorems \ref{thm:uap} and \ref{example} combined suggest that the number of nonzero parameters in our network for approximating functions in $C([a,b]^d)$ is optimal in the sense that it grows linearly with the input dimension $d$. 
	
	\section*{Acknowledgment}
	M.~M. was partially supported by NSERC Postdoctoral Fellowship. H.~Y. was partially supported by the US National Science Foundation under awards DMS-2244988, DMS-2206333, the Office of Naval Research Award N00014-23-1-2007, and the DARPA D24AP00325-00.

\end{document}